\documentclass[twoside]{article}

\usepackage[accepted]{aistats2023}

\usepackage{empheq,amsthm,amssymb}
\usepackage[all,warning]{onlyamsmath}

\usepackage{mathrsfs}

\usepackage{xcolor}
\usepackage{graphicx,caption,subcaption}
\graphicspath{{FIGS/}}

\usepackage{url}
\usepackage{hyperref}
\hypersetup{
	colorlinks = true, 
	urlcolor = blue, 
	linkcolor = purple, 
	citecolor = blue 
}
\usepackage{booktabs}

\usepackage[capitalize,nameinlink]{cleveref}

\usepackage{flushend}

\usepackage[round,compress]{natbib}

\numberwithin{equation}{section}

\newtheorem{theorem}{Theorem}
\newtheorem{corollary}{Corollary}

\newtheorem{assumption}{Assumption}

\usepackage{glossaries}
\setacronymstyle{long-sc-short}
\glsdisablehyper

\newacronym{ENKF}{enkf}{ensemble Kalman filter}
\newacronym{EKF}{ekf}{extended Kalman filter}
\newacronym{UKF}{ukf}{unscented Kalman filter}
\newacronym{SSM}{ssm}{state-space model}
\newacronym{RNN}{rnn}{recurrent neural net}

\DeclareMathOperator*{\argmin}{\arg\!\min}

\begin{document}

\twocolumn[

\aistatstitle{Recurrent Neural Networks and Universal Approximation of Bayesian Filters}

\aistatsauthor{ Adrian N. Bishop \And Edwin V. Bonilla}

\aistatsaddress{ University of Technology Sydney \And  CSIRO's Data61 } ]

\begin{abstract}
We consider the Bayesian optimal filtering problem: i.e. estimating some conditional statistics of a latent time-series signal from an observation sequence. Classical approaches often rely on the use of assumed or estimated transition and observation models. Instead, we formulate a generic recurrent neural network framework and seek to learn directly a recursive mapping from observational inputs to the desired estimator statistics. The main focus of this article is the approximation capabilities of this framework. We provide approximation error bounds for filtering in general non-compact domains. We also consider strong time-uniform approximation error bounds that guarantee good long-time performance. We discuss and illustrate a number of practical concerns and implications of these results.
\end{abstract}

\section{INTRODUCTION}
\label{sec:introduction}

Optimal filtering \citep{BainCrisan2008a} is concerned with estimating some statistics of a latent random signal (or state) ${X}_t$ at the current time $t\in\mathbb{N}$, conditioned on some observations ${Y}_\tau$ collected thus far, i.e., $0\leq \tau\leq t$. When the signal transition and observation models are linear with additive Gaussian noise, the solution is given by the celebrated Kalman filter \citep{Kalman1960a}. In general nonlinear, non-Gaussian, settings, there is no tractable finite-dimensional optimal filter and approximations are needed. Nonlinearity and high state-dimensionality, typical in applications, make the filtering problem challenging. 

We consider a data-driven approach to learning an optimal estimator. More specifically, we consider scenarios in which latent signal and observation data may be collected by computer simulations (if a good model is available), actual experiments (otherwise) or some combination thereof. We provide a framework for optimal filtering using general \gls{RNN} structures that may be trained by minimizing the empirical mean-square error between the sampled signal data and the network output, with the network taking as input only the current observation at each time. The output of the \gls{RNN} acts as a recursive state estimator. We study a number of interesting approximation capabilities of this estimator with respect to the true optimal nonlinear filter (e.g. the true conditional expectations).

\textbf{Summary of contributions}: We propose a generic \gls{RNN}-based architecture and methodology for optimal (Bayesian) filtering in general state-space models. This formulation is kept generic so as to facilitate a study of the approximation capabilities and limitations of \gls{RNN}s in applications to Bayesian filtering. Firstly: We show that a generic \gls{RNN}-based estimator can approximate the optimal estimate of the signal to any desired accuracy on a fixed finite time interval of interest. We note that almost no assumptions on the state-space model are needed in this case. Furthermore, the model may not even be known if experimental data is available to generate signal training data. However, the approximation capability in this setting relies strictly on the assumption of a finite horizon of interest. Secondly: We study the time-uniform universal approximation of Bayesian filters with \gls{RNN}s. For a particular class of models, we show the approximation error can be bounded, to any desired accuracy, uniformly \textit{for all time}; i.e. implying that approximation errors do not accumulate over time. This result has important practical implications, e.g. it may influence network design, and it allows one to train on a signal-observation sequence of (short) finite length while permitting the filtering algorithm to run indefinitely (as is typical in applications). Thirdly: We discuss and illustrate a number of practical consequences of both results and contrast these with each other and with other methodologies.

\label{sec:related-work}
\textbf{Background and related work in filtering}: 
Optimal filtering and related problems in learning and inference in dynamical systems are of interest across many fields of study, including control and signal processing \citep{AndersonMoore1979a,BainCrisan2008a}, geophysics \cite{Evensen2009a}, machine learning \citep{van2000unscented,ghahramani1995factorial}, and statistics \citep{PittShephard1999a,andrieu2010particle}. A comprehensive review of the methodology in this area is beyond the scope of this paper. We simply note in passing some popular model-based approximations for nonlinear filtering such as the \gls{EKF} \citep{AndersonMoore1979a} and the \gls{UKF} \citep{julier-2004}, and Monte Carlo integration methods for filtering, termed particle filters \citep{GordonSalmondSmith1993a,PittShephard1999a}, and \gls{ENKF} \citep{Evensen2009a} methods. There are some adaptions of these model-based methods to data-driven, likelihood-free, filtering, e.g., employing approximate Bayesian computation within Monte Carlo \citep{JasraSinghMartinEtAl2010a,MartinMcCabeFrazierEtAl2019a} or using Gaussian processes \citep{KoFox2009a}. 

Neural network approaches to optimal nonlinear filtering were considered in early work in \cite{Lo1994a,ParisiniZoppoli1994a,V.T.Shin1994a,AlessandriBagliettoParisiniEtAl1999a,ParlosMenonAtiya2001a}. The article \cite{HaykinYeeDerbez1997a} provides a good early summary and survey. These works naturally considered very simple networks, e.g. one layer and just a handful of sigmoidal-type activation functions; e.g. see \cite{Lo1994a}. However, convergence and approximation error results were also given in \cite{Lo1994a, ParisiniZoppoli1994a} and also in \cite{ParisiniAlessandriMaggioreEtAl1997a,AlessandriBagliettoParisiniEtAl1999a}. We generalize, strengthen, and add to these early results in a number of ways in this work; e.g. in particular with our main result on time-uniform approximation bounds with observations on non-compact domains.

Other work has used machine learning to improve the optimization of calculating the initial state in a variational data assimilation framework \citep{frerix2021variational} and inference in state-space models \citep[see, e.g.,][Ch.~18]{murphy2012machine}. Some approaches have developed approximate inference techniques in parametric settings \citep{ghahramani2000variational,fox2008nonparametric}. Later methods have used non-parametric models such as Gaussian processes \citep{frigola2014variational,doerr2018probabilistic,nickisch2018state,pmlr-v97-ialongo19a} or flexible modern neural network-based frameworks \citep{KrishnanShalitSontag2017a,gu2015neural,KarlSoelchBayerEtAl2016a,haarnoja2016backprop,becker2019recurrent}. Finally, we note that learning time-series models with neural networks is closely related to the filtering discussed here; e.g., it may be considered a special case in which the observed data is just the signal process with no noise, see, e.g., \cite{KarlSoelchBayerEtAl2016a,RangapuramSeegerGasthausEtAl2018a}. 

The purpose of this work is not the development of new (\gls{RNN}-based) methodology for Bayesian filtering. Instead, we study the universal approximation capability of (rather generic) \gls{RNN}-based approximations of the optimal filter, both on finite time intervals and uniformly in time. In later sections, when we can easily reference specific technical details and conditions, we discuss the results of this article, and (briefly) contrast these with other approximation methods as in \cite{HeineCrisan2008a,Handel2009c,Whiteley2013a,DoucMoulinesOlsson2014a,CrisanLopezYelaMiguez2020a}.

\section{DISCRETE-TIME BAYESIAN FILTERING}

Let $\mathbb{X}\subseteq\mathbb{R}^{d_x}$ and $\mathbb{Y}\subseteq\mathbb{R}^{d_y}$ with the Borel $\sigma$-algebra $\mathcal{B}(\mathbb{X})$. Consider a Markov chain $(X_t)_{t\in\mathbb{N}}$ taking values in $\mathbb{X}$ with Markov kernel $\mathsf{K} :\mathbb{X} \times \mathcal{B}(\mathbb{X})\rightarrow [0,1]$. Consider a process $(Y_t)_{t\in\mathbb{N}}$ defined on $\mathbb{Y}$, conditionally independent given $(X_t)_{t\in\mathbb{N}}$, with a transition density $g:\mathbb{X}\times \mathbb{Y}\rightarrow \mathbb{R}$ with respect to the Lebesgue measure. The process $X_t$ is thought of as being observed via the process $Y_t$. The process $(Y_t)_{t\in\mathbb{N}}$ is itself not a Markov chain, but the pair $(X_t,Y_t)_{t\in\mathbb{N}}$ is a Markov chain on $\mathbb{X}\times \mathbb{Y}$. Denote\footnote{Superscripts often denote initial conditions for random processes, but are also indices (e.g. over data, parameters, etc), and also just powers in some places. The use and case should be clear.} by $\mathbf{P}^\mu$ the law of $(X_t,Y_t)_{t\in\mathbb{N}}$ under which the pair $(X_t,Y_t)_{t\in\mathbb{N}}$ is a Markov chain on $\mathbb{X}\times \mathbb{Y}$ with $X_0$ having measure $\mu\in\mathcal{P}(\mathbb{X})$. The space of all probability measures on $\mathbb{X}$ is $\mathcal{P}(\mathbb{X})$. Expectation with respect to $\mathbf{P}^\mu$ is denoted by $\mathbf{E}^\mu$.

Filtering involves computing the regular conditional distribution $\pi_t^\mu(A) := \mathbf{P}^\mu(X_t\in\,A\,|\,Y_1,\ldots,Y_t)$, for all $A\in\mathcal{B}(\mathbb{X})$. The distribution $\pi_t^\mu\in\mathcal{P}(\mathbb{X})$ is called the filtering distribution. From Bayes' rule,
\begin{align}
\pi_0^{\mu} \,&=\, \mu,\qquad\nonumber\\
\pi^{\mu}_t \,&=\, \Psi(\pi^{\mu}_{t-1},Y_t) \label{truebayesfilter}
\end{align}
where, for some $\nu\in\mathcal{P}(\mathbb{X})$ and $y\in\mathbb{Y}$, the function $\Psi:\mathcal{P}(\mathbb{X})\times\mathbb{Y}\rightarrow\mathcal{P}(\mathbb{X})$ is defined by,
\begin{equation}
\Psi(\nu,y)(A) :=\frac{\int_A \, g(x,y) \,(\nu\mathsf{K})(dx) }{\int  g(x,y) \,(\nu\mathsf{K})(dx)},~~ \forall \,A\in\mathcal{B}(\mathbb{X})
\end{equation}
where $(\nu\mathsf{K})(dx):=\int\mathsf{K}(x',dx)\,\nu(dx')$. The mapping $\Psi$ is time-invariant. The process $(\pi_t^\mu)_{t\in\mathbb{N}}$ on $\mathcal{P}(\mathbb{X})$ is a Markov process when $X_0$ has measure $\mu\in\mathcal{P}(\mathbb{X})$ \citep{Stettner1989a}.

For a finite-dimensional integrable function $x\mapsto\rho(x)$, $x\in\mathbb{X}$, taking values in some Euclidean space, the filtering problem is often stated in terms of a point-valued estimate,
\begin{equation}
	\pi_t^\mu(\rho) := \int\rho(x)\pi_t^\mu(dx) = \mathbf{E}^\mu[ \rho(X_t) \,|\, Y_1,\ldots,Y_t] \label{filteringpointestimate}
\end{equation}
For example, one may want to estimate just the conditional mean and covariance of $X_t$, if they exist.

\subsection{Bayesian Filtering as Optimal State Estimation}

A closely related problem to that of filtering is state estimation. Going forward, we assume $(X_t,Y_t)$ are jointly square integrable and any measurable function preserves square integrability, e.g. $\mathbb{E}[\|\rho(X_{t})  \|^2]<\infty$. Then, consider the measurable function $\overline{\rho}^\mu_{t}:=\overline{\rho}^\mu_{t}(Y_1,\ldots,Y_t)$; being the solution of the optimization problem,
\begin{equation}
	\overline{\rho}^\mu_{t} ~:=~ \argmin_{f}~\mathbf{E}^\mu\left[ \| f(Y_1,\ldots,Y_t) -\rho(X_{t})  \|^2 \right]\label{optimalcriterion}
\end{equation}
where $f$ is any measurable function of $(Y_1,\ldots,Y_t)$. The solution to (\ref{optimalcriterion}) is related to the optimal filtering problem by $\overline{\rho}^\mu_{t} = \pi_t^\mu(\rho)$, as defined in (\ref{filteringpointestimate}). Note $\overline{\rho}^\mu_{t}$ is square integrable (as it may be viewed as the orthogonal projection onto the closed subspace of square integrable functions). 

We think of $\overline{\rho}^\mu_{t}$ as the optimal estimate of the state $\rho(X_{t})$. This estimate is intractable in all but the most pathological situations. It is desirable in practice to design an approximation of $\overline{\rho}^\mu_{t}$. We denote a computable estimate by $\widehat{\rho}^\nu_{t}$ where $\nu\in\mathcal{P}(\mathbb{X})$ is a known, assumed, or estimated distribution of ${X}_0$ such that $\widehat{\rho}^\nu_{0}:=\nu(\rho)$ is also computable.

We emphasise here that the solution to (\ref{optimalcriterion}), as given by $\overline{\rho}^\mu_{t} = \pi_t^\mu(\rho)$, see (\ref{filteringpointestimate}), is not generally given in the form of a recursion, even though the recursion (\ref{truebayesfilter}) on the level of the conditional distribution exists.

To close this section, note $\pi_t^\mu(\rho)$ makes sense in more general settings than outlined for $(X_t,Y_t)$ thus far (e.g. beyond independent, time-invariant, Markov settings); as does the solution to (\ref{optimalcriterion}). Recursions of the form (\ref{truebayesfilter}) also hold under more general models, see \cite{TongHandel2012a}.

\subsection{Sufficient Coordinates and Recursive Filtering}

The recursion in (\ref{truebayesfilter}) is on the level of probability measures and is generally infinite dimensional and intractable. Informally, e.g. neglecting existence and other structure, we may define a (generally infinite dimensional) vector state of sufficient statistics on a (separable) Banach space $\mathbb{S}$. For the filtering distribution $\pi_t^\mu\in\mathcal{P}(\mathbb{X})$ we may consider, for example, the conditional moments,
\begin{equation}
	S^\mu_t := \int  \begin{bsmallmatrix} x \\ \mathrm{vec}(xx^\top)\\ \vdots\end{bsmallmatrix}\,\pi_t^\mu(dx)
\end{equation}
as in \cite{Rudenko2010a}. A recursive Bayes filter may then take the form,
\begin{align}
	S^\mu_t \,=\,{\Psi}({S}^\mu_{t-1},Y_{t})  \,=:\,{\Psi}_t({S}^\mu_{t-1}) \label{suffstat-filter}
\end{align}
on $\mathbb{S}$ for some measurable function ${\Psi}:\mathbb{S}\times\mathbb{Y}\rightarrow\mathbb{S}$. The subscript $t$ in ${\Psi}_t$ succinctly indexes the observational input; highlighting the recursion on $\mathbb{S}$. In this notation, we write the composition $S^\mu_t ={\Psi}_t{\Psi}_{t-1}\cdots{\Psi}_{\tau+1}({S}^\mu_{\tau})$ for $\tau<t\in\mathbb{N}$. If such a representation exists, then ${\Psi}$ is time-invariant. 

We may claim that a finite-dimensional recursive optimal filter exists if and only if $\mathbb{S}$ is finite dimensional and we can transform between $S^\mu_t$ and $\pi_t^\mu$ in some standard way. In the linear-Gaussian setting a finite-dimensional state $S^\mu_t$ and recursion (\ref{suffstat-filter}) exists from which $\pi_t^\mu$ follows \citep{AndersonMoore1979a}. See also \cite{Sawitzki1981a, DickinsonSontag1985a, LevinePignie1986a}. If $S^\mu_t $ is finite, then it is a special case of $\pi_t^\mu(\rho)$ for some $x\mapsto\rho(x)$.

\section{GENERIC RNN-BASED STATE ESTIMATORS}

We introduce a generic \gls{RNN} architecture for state estimation. Consider a network with $L\in\mathbb{N}$ layers, in the form,
\begin{align}
	\widehat{\rho}^\nu_t &= \mathbf{W}_{(L,L-1)}\mathbf{s}_{t}^{(L-1)} + \mathbf{b}_{(L)} \label{generalRNN1} \\
	\mathbf{s}_{t}^{(l)} &= \sigma\left(\mathbf{W}_{(l,l-1)}\mathbf{s}_{t}^{(l-1)} + \mathbf{b}_{(l)} + \textstyle{\sum_{k=l}^{L-1}}\,\mathbf{W}_{(l,k)}\mathbf{s}_{t-1}^{(k)}\right),\nonumber\\
	&\qquad l\in\{1,\ldots,L-1\} \label{generalRNN2} \\
	\mathbf{s}_{t}^{(0)} &:= Y_t \label{generalRNN3}
\end{align}
where $\sigma(\cdot)$ is a vector-valued activation function acting component-wise on its argument. Component-wise, i.e. on scalar inputs, we assume $\sigma(\cdot):=\max\{0,\cdot\}$ here, i.e. we consider only rectified linear activation units (\textsc{ReLU}s). The parameter $\theta := [(\mathbf{W}_{(l,l-1)}, \mathbf{W}_{(l,k)}), \mathbf{b}_{(l)}]$ consists of appropriately sized real matrices and vectors. The generic notation for the \gls{RNN}-approximated state estimator is,
\begin{equation}
	\widehat{\rho}^\nu_t \,=\, \widehat{\Psi}^\theta(\mathbf{s}_{t-1}, Y_t) \,:=\, \widehat{\Psi}^\theta_t(\mathbf{s}_{t-1}) \label{RNNsuccinct} 
\end{equation}
where $\mathbf{s}_t$ is a stacked vector of the $\mathbf{s}_{t}^{(l)}$. The parameter $\theta$ is time-invariant and takes values in a high-dimensional Euclidean space. The recursion (\ref{RNNsuccinct}) is on the level of the internal network state $\mathbf{s}_{t-1}$ which we suggest captures in some sense a finite-dimensional approximation of sufficient statistics as in (\ref{suffstat-filter}), enough to approximate $\pi_t^\mu(\rho)$. 

\begin{assumption}\label{assump:basic}
	The process $(X_t,Y_t)_{t\in\mathbb{N}}$ is jointly square integrable and $x\mapsto\rho(x)$, and $y\mapsto\widehat{\Psi}^\theta(\cdot,y)$, are finite-dimensional, square integrable, and take values in some Euclidean space. The assumption on $\widehat{\Psi}^\theta$ is provable.
\end{assumption}

Common specialisations of  (\ref{generalRNN1}), (\ref{generalRNN2}), (\ref{generalRNN3}) involve feedback solely from the last activation layer to the first layer; and feedback from the output of each layer to itself, see \cite{PascanuGulcehreChoEtAl2014a}. Our results are proven under common simplifications of  (\ref{generalRNN1}), (\ref{generalRNN2}), (\ref{generalRNN3}).

The ideal \gls{RNN}-approximated estimator is based on optimizing $\theta$ according to the following cost functional,
\begin{equation}
	\mathcal{C}(\theta) \,:=\, \frac{1}{T} \, \sum_{t=1}^T \, \mathbf{E}^\mu\left[ \|\widehat{\Psi}^\theta(\mathbf{s}_{t-1}, Y_t) -\rho(X_{t}) \|^2 \right] \label{optimalcriterionneural}
\end{equation}
on a finite horizon $T\in\mathbb{N}$. In practice, suppose some data $\mathfrak{D}_{T,N}:= \smash{(X_t^{(n)},Y_t^{(n)})}$, $t\in\{1,\ldots,T\}$, $n\in\{1,\ldots,N\}$ with finite $N\in\mathbb{N}$, is independently sampled according to the law $\mathbf{P}^\mu$. The data may be collected via computer simulations (if a good model $\mathbf{P}^\mu$ is available), real-world experiments, or some combination. Practically, a network is trained by minimizing, over $\theta$, the empirical loss,
\begin{align}
	\theta^* &:= \argmin_{\theta}~\mathcal{C}^N(\theta), \label{empiricalcostneural}\\
	\mathcal{C}^N(\theta) &:=\frac{1}{N T} \, \sum_{n=1}^N\sum_{t=1}^T \, \|\widehat{\Psi}^\theta(\mathbf{s}_{t-1}, Y_t^{(n)}) -\rho(X^{(n)}_{t}) \|^2 \nonumber
\end{align}
with $\mathbf{s}_{0}$ given. The loss $\mathcal{C}^N(\theta)$ is a sample version of (\ref{optimalcriterionneural}) and $\mathcal{C}^N(\theta)\longrightarrow_{N\rightarrow\infty}\mathcal{C}(\theta)$. In practice, we might define $\widehat{\rho}^\nu_0 :=\nu(\rho)$ and set $\mathbf{s}_{0}$ commensurately, dependent on the network. We may also consider $\mathbf{s}_{0}$ as a model parameter and optimise its value \citep{Lo1994a}. The approximation theorems consider the initialisation as part of the result.

\section{APPROXIMATION THEOREMS}

\subsection{A General Approximation Theorem}

\begin{theorem} \label{theorem:basicapprox}
	Let Assumption \ref{assump:basic} hold. Let $\widehat{\Psi}^\theta(\mathbf{s}_{t-1}, Y_t)$ denote a generic multilayer \gls{RNN}, as in (\ref{generalRNN1}), (\ref{generalRNN2}), (\ref{generalRNN3}), taking as input elements in the sequence $(Y_t)_{t\in\{1,\ldots,T\}}$. Then for any $\epsilon>0$, there exists a finite real parameter vector $\theta^*$, and an initialisation vector $\mathbf{s}_{0}$, such that,
\begin{equation}
	\frac{1}{T} \, \sum_{t=1}^T \, \mathbf{E}^\mu\left[ \| \pi_t^\mu(\rho) - \widehat{\Psi}^{\theta^*}(\mathbf{s}_{t-1}, Y_t)  \|^2 \right]^{\frac{1}{2}} \,\leq\, \epsilon \label{optimalapprox1}
\end{equation}
\end{theorem}
\begin{proof}
Note $\pi_t^\mu(\rho) = \mathbb{E}\left[ \rho(X_t) \,|\, Y_1,\ldots,Y_t\right]$. This is an extension of \cite{Lo1994a}, who considers single-layer networks with observations on compact domains. Instead we use \cite{KidgerLyons2020a} and consider deep \gls{RNN}s on non-compact domains. We construct a network of the form,
\begin{align}
	\widehat{\rho}^\nu_t &= \mathbf{W}_{(L,L-1)}\mathbf{s}_{t}^{(L-1)} + \mathbf{b}_{(L)} \label{neuralproofstruct1} \\
	\mathbf{s}_{t}^{(l)} &= \sigma(\mathbf{W}_{(l,l-1)}\mathbf{s}_{t}^{(l-1)}  + \mathbf{b}_{(l)}),~~\nonumber\\
	&\qquad\qquad\qquad\qquad\qquad l\in\{2,\ldots,L-1\} \label{neuralproofstruct2} \\
	\mathbf{s}_{t}^{(1)} &= \sigma(\mathbf{W}_{(1,0)}Y_t + \mathbf{W}_{(1,1)}\mathbf{s}_{t-1}^{(1)}  + \mathbf{b}_{(1)}) \label{neuralproofstruct3}
\end{align}
with $\mathbf{W}_{(2,1)}=\mathbf{W}_{(2,1.5)}\mathbf{W}_{(1.5,1)}$ and $\mathbf{b}_{(2)}=\mathbf{b}_{(2)}'+\mathbf{b}_{(1.5)}$. Consider the first layer and a half,
\begin{align}
	\mathbf{s}^{(1.5)}_t &\,=\, \mathbf{W}_{(1.5,1)}\mathbf{s}_{t}^{(1)}  + \mathbf{b}_{(1.5)} 
\end{align}
with $\mathbf{s}^{(1.5)}_t$ being the effective input then to layer $2$. We seek at any time $t\in\{1,\ldots,T\}$ the state,
\begin{equation}
	\mathbf{s}^{(1.5)}_t \,=\, \left[t~{\overline{\rho}^\mu_0}^\top~Y_t^\top~ Y_{t-1}^\top~ \ldots Y_1^\top~ 0~ 0~ \ldots~ 0\right]^\top \label{halflayerdesiredstate}
\end{equation}
with $\mathbf{s}^{(1.5)}_T  = [T\, {\overline{\rho}^\mu_0}^\top\, Y_T^\top\,  \ldots\, Y_1^\top ]$. Let $\mathbf{W}_{(1.5,1)}=\mathbf{I}$ and,
\begin{align}
	\mathbf{W}_{(1,0)}&= \begin{bsmallmatrix} \left[0~0~\cdots~0\right]\\ 
									\mathbf{0}\\
									\mathbf{I} \\
										 \mathbf{0} \\
										  \vdots  \\
										  \mathbf{0} \end{bsmallmatrix}, \\
	\mathbf{W}_{(1,1)}&= \begin{bsmallmatrix} 1 &0~0 & 0~0 & 0~0 & \cdots & 0~0 & 0~0\\
										0 & \mathbf{I} & \mathbf{0} & \mathbf{0} & \mathbf{0} & \cdots & \mathbf{0}\\
										0 & \mathbf{0}& \mathbf{0} & \mathbf{0} & \mathbf{0} & \cdots & \mathbf{0}\\
										0 &  \mathbf{0}& \mathbf{I} & \mathbf{0} & \mathbf{0} & \cdots & \mathbf{0}\\
										 0 &  \mathbf{0} & \mathbf{0}&\mathbf{I} & \mathbf{0} & \cdots & \mathbf{0}\\
										 0 &  \mathbf{0}& \mathbf{0} &\mathbf{0} & \ddots & \ddots & \mathbf{0}  \end{bsmallmatrix}
\end{align}
where $\mathbf{I}$ in $\mathbf{W}_{(1,0)}$ and all but the first $\mathbf{I}$ in $\mathbf{W}_{(1,1)}$ denotes a $d_y$-dimensional identity matrix, and the first $\mathbf{I}$ in $\mathbf{W}_{(1,1)}$ is the size of $\overline{\rho}^\mu_0$. And let,
\begin{align}
	 \mathbf{b}_{1} &\,=\, \left[ 1~ 
									[0\,\cdots\,0]^\top ~
									[b\,\cdots\,b]^\top ~
									0 ~
									\cdots ~
								{0} \right]^\top, \\
	 \mathbf{b}_{(1.5)} &\,=\, \left[\begin{array}{cccccccc} 0& 
									{-b} &
									{-b} &
									\cdots &
								{-b}  \end{array}\right]^\top
\end{align}
Initialise,
\begin{align}
 	\mathbf{s}_{0} \,&=\, \mathbf{s}_{0}^{(1)}\nonumber\\
	 \,&=\, \left[0~ [{\overline{\rho}^\mu_0}+[b\,\cdots\,b]^\top]^\top \, [0\,\cdots\,0]\, b\, \ldots\, b \right]^\top
\end{align}
With this network construct and with $b>0$ large enough we find $\forall t\in\{1,\ldots,T\}$ the state $\smash{\mathbf{s}^{(1.5)}_t}$ is given by (\ref{halflayerdesiredstate}) over a domain on which $\mathbf{P}^\mu$ places most mass, as desired, see \citet[proof of Theorem 4.16]{KidgerLyons2020a}.

There is no more feedback in the network constructed in this proof and the state $\mathbf{s}^{(1.5)}_t$ can thus be viewed as an input for the feedforward neural network from layers $2$ to $L$. 

We define a target function, for the feedforward network from layers $2$ to $L$, by,
\begin{align}
	& f(t, \overline{\rho}^\mu_0, Y_1,\ldots,Y_T) \nonumber\\
	 &~~= \left\{\begin{array}{ll} \overline{\rho}^\mu_0  & if~t=0\\
								\mathbf{E}^\mu\left[ \rho(X_1) \,|\, Y_1 \right]  & if~t=1 \\
													\qquad \vdots   & \\
								\mathbf{E}^\mu\left[ \rho(X_T) \,|\, Y_T, \cdots, Y_1\right]  & if~t=T\\
								\end{array} \right. \label{finitenewtargetf}
\end{align}
This function is (Borel) measurable. We now apply classical universal approximation theorems on $f$. See \citet[proof of Theorem 4.16]{KidgerLyons2020a} for an easy to follow construction immediately applicable here. 
\end{proof}

The preceding theorem holds only on those finite time horizons $t\in\{1,\ldots,T\}$, owing to the special structure of the constructed network in the proof. The finite time $T\in\mathbb{N}$ may be arbitrary, but the size of the resulting network (in theory) grows with $T$. We discuss further the results of this work later. Next we show that, for a particular class of models, the approximating network size is not generally a function of time and the approximation error remains bounded uniformly \textit{for all time}; i.e. implying that approximation errors do not accumulate over time.

\subsection{Recursive Filters, Approximations, and Time-Uniform Approximation Error Bounds} \label{timeuniformdiscsection}

The main result is presented in this subsection. We consider a finite-dimensional statistic of interest $\pi_t^\mu(\rho)=:S_t^\mu$ that evolves recursively as per (\ref{suffstat-filter}). We may consider a finite-dimensional truncation of an infinite-dimensional sufficient statistic in place of $S_t^\mu$ if necessary. See \cite{Sawitzki1981a, DickinsonSontag1985a} for general finite-dimensional filtering results and \cite{Goodman1975a,Segall1976b,Marcus1979a,Daum1986b} for models, which may lead naturally to finite-dimensional truncated representations of statistics. We need some assumptions.

\begin{assumption}\label{assump:stationary}
Observation process $(Y_t)_{t\in\mathbb{N}}$ is stationary. 
\end{assumption}

\begin{assumption}\label{assump:stability}
The map $s\mapsto \Psi(s,\cdot)$ is Lipschitz with finite Lipschitz constant, and there exist real finite constants $C>0$, $0<\kappa<1$ independent of $\tau,t\in\mathbb{N}$ such that,
\begin{align}
	\mathbf{E}^\mu \left[ \,\left\| \mathsf{\Psi}_{t}\cdots\mathsf{\Psi}_{\tau}({S}^{s_0}_{\tau-1}) - \mathsf{\Psi}_{t}\cdots\mathsf{\Psi}_{\tau}({S}^{s}_{\tau-1}) \right\|^2 \,\right]^{\frac{1}{2}} 
	  \nonumber \\
	  \,\leq\, C\,\kappa^{(t-\tau)}\, \mathbf{E}^\mu\big[ \|{S}_{\tau-1}^{s_0} - {S}_{\tau-1}^s \|^2 \big]^{\frac{1}{2}} \label{stabilityresult1}
\end{align}
for any initial points $s_0,s\in\mathbb{S}$ and all $\tau\leq t$.
\end{assumption}

This (\ref{stabilityresult1}) is an ``\textit{eventually (exponentially) contracting on average}''-type property of the optimal filter (\ref{suffstat-filter}). With (\ref{stabilityresult1}), trajectories eventually come together on average, at an exponential rate; but they may never converge to a fixed point or invariant measure. The latent signal need not be ergodic nor admit an invariant measure for this condition to hold; but under general conditions if the signal is (exponentially) ergodic, then one expects so is the Bayesian filter (\ref{truebayesfilter}) or (\ref{suffstat-filter}). Results of this type (but often not as strong as the assumption) have been widely studied, and we reference: \cite{AtarZeitouni1997b,BudhirajaOcone1999a,ChiganskyLiptser2004a,OudjaneRubenthaler2005a,KleptsynaVeretennikov2007a,CrisanHeine2008a,Handel2008a,TongHandel2012a,GerberWhiteley2017a}. Although the details differ and are challenging, known results support the intuition that sufficiently informative observations or sufficiently contractive latent signals may generally translate to results like in Assumption \ref{assump:stability}. This idea is well exemplified with linear-Gaussian models, discussed later, see \cite{BishopDelMoralStability}. The following is the main result.

\begin{theorem}\label{theo:maintimeuniform}
	Let $\widehat{\Psi}^\theta(\mathbf{s}_{t-1}, Y_t)$ denote a generic multilayer \gls{RNN}, as in (\ref{generalRNN1}), (\ref{generalRNN2}), (\ref{generalRNN3}), taking as input elements in the sequence $(Y_t)_{t\in\mathbb{N}}$. Suppose Assumptions \ref{assump:basic}, \ref{assump:stationary} and \ref{assump:stability} hold. Then for any $\epsilon>0$, there is a finite real parameter $\theta^*$, and an initialisation vector $\mathbf{s}_{0}$, such that,
\begin{equation}
	 \mathbf{E}^\mu\left[\, \left\| {S}^{s_0}_{t} - \widehat{\Psi}^{\theta^*}(\mathbf{s}_{t-1}, Y_t) \right\|^2 \right]^{\frac{1}{2}} \,\leq\, \epsilon \label{timeuniformstatement1}
\end{equation}
for all $s_0\in\mathbb{S}$ and all $t\in\mathbb{N}$. This is a time-uniform bound.
\end{theorem}

\begin{proof}
We construct a \gls{RNN} in the special form,
\begin{align}
	\widehat{S}^{\mathbf{s}_0}_t &\,=\, \mathbf{W}_{(L,L-1)}\mathbf{s}_{t}^{(L-1)} + \mathbf{b}_{(L)} \label{RNNrecursiveform1} \\
	\mathbf{s}_{t}^{(l)} &\,=\, \sigma(\mathbf{W}_{(l,l-1)}\mathbf{s}_{t}^{(l-1)}  + \mathbf{b}_{(l)}),  \nonumber\\
	& \qquad\qquad\qquad\qquad l\in\{2,\ldots,L-1\} \label{RNNrecursiveform2} \\
	\mathbf{s}_{t}^{(1)} &\,=\, \sigma(\mathbf{W}_{(1,0)}Y_t + \mathbf{W}_{(1,L-1)}\mathbf{s}_{t-1}^{(L-1)}  + \mathbf{b}_{(1)}) \label{RNNrecursiveform3}
\end{align}
Only the output of the last activation layer is fed back to the input of the first layer. With no loss of generality, we write,
\begin{equation}
	\widehat{S}^{\mathbf{s}_0}_t \,=\, \widehat{\mathsf{\Psi}}_{t}(\mathbf{s}_{t-1}) \,=\, \widehat{\mathsf{\Psi}}_{t}(\widehat{S}^{\mathbf{s}_0}_{t-1}) \label{networknotation}
\end{equation}
The notation $\widehat{S}_t =\widehat{\mathsf{\Psi}}_{t}(\widehat{S}_{t-1})$ and the network structure highlights the recursive nature of this approximation, in line with (\ref{suffstat-filter}). This feedback structure is quite different to that used in the proof of Theorem \ref{theorem:basicapprox}, as discussed later.

Consider the error,
\begin{equation}
	S_t^{s_0}  - \widehat{S}_t^{\mathbf{s}_{0}}  = {\Psi}_{t}\cdots{\Psi}_2 {\Psi}_1(s_0) - \widehat{\Psi}_{t}\cdots\widehat{\Psi}_2\widehat{\Psi}_1(\mathbf{s}_0) 
\end{equation}
for any ${s}_0\in\mathbb{S}$ and with $\mathbf{s}_0=\mathbf{s}_{0}^{(L-1)}$ in this case. Let $\mathbf{s}_{0}^{(L-1)}=s_0$. We expand this error as a telescopic sum as,
\begin{align}
	 S_t^{s_0}  - \widehat{S}_t^{s_{0}}  &= \sum_{\tau=1}^{t} \Big(\widehat{\Psi}_{t}\cdots \widehat{\Psi}_{\tau+1}  {\Psi}_{\tau}\big({\Psi}_{\tau-1}\cdots\left({s}_0\right)\big) \nonumber \\
	&~~~~ -\, \widehat{\Psi}_{t}\cdots \widehat{\Psi}_{\tau+1}  \widehat{\Psi}_{\tau}\big({\Psi}_{\tau-1}\cdots\left({s}_0\right)\big) \Big)
\end{align}
This expansion formula is easy to check, e.g. try $t\in\{3,4\}$. Note $S_{\tau-1}^{s_0}=\Psi_{\tau-1}\cdots(s_0)$, and we then have,
\begin{align}
	  \mathbf{E}^\mu\left[\, \big\|  S_t^{s_0}  \,-\, \widehat{S}_t^{{s}_0} \big\|^2\,\right]^{\frac{1}{2}} =\, \qquad\qquad\qquad\qquad\quad \nonumber \\
		\leq~ \sum_{\tau=1}^{t}    \mathbf{E}^\mu\Big[\, \big\| \widehat{\Psi}_{t}\cdots \widehat{\Psi}_{\tau+1} {\Psi}_{\tau}\big(S_{\tau-1}^{s_0} \big)\qquad ~ \nonumber \\
		- \widehat{\Psi}_{t}\cdots \widehat{\Psi}_{\tau+1}  \widehat{\Psi}_{\tau}\big(S_{\tau-1}^{s_0}\big) \big\|^2\,\Big]^{\frac{1}{2}}
\end{align}
Assume now there exists real finite constants $C>0$, $0<\kappa<1$ independent of $\tau\leq t\in\mathbb{N}$ such that,
\begin{align}
	\mathbf{E}^\mu \left[ \,\left\| \widehat{\Psi}_{t}\cdots\widehat{\Psi}_{\tau}(\widehat{S}^{s_0}_{\tau-1}) - \widehat{\Psi}_{t}\cdots\widehat{\Psi}_{\tau}(\widehat{S}^{s}_{\tau-1}) \right\|^2 \,\right]^{\frac{1}{2}}  \nonumber\\
	  \leq\, C\,\kappa^{(t-\tau)}\, \mathbf{E}^\mu \big[ \|\widehat{S}_{\tau-1}^{s_0} - \widehat{S}_{\tau-1}^s \|^2 \big]^{\frac{1}{2}} \label{stabilityresult2}
\end{align}
for any initial points $s_0,s\in\mathbb{S}$ and all $\tau,t\in\mathbb{N}$. That is, we assume the approximated filter inherits the hypothesised stability of the optimal filter (maybe with different constants, but the specifics of the constants won't matter). We verify this assumption later. Applying this condition,
\begin{align}
		\mathbf{E}^\mu \left[ \, \big\| S_t^{s_0}  \,-\, \widehat{S}_t^{s_{0}}\big\|^2 \,\right]^{\frac{1}{2}} \qquad\qquad\qquad\qquad\qquad\qquad\qquad  \nonumber\\
		\,\leq\,  \mathbf{E}^\mu \left[ \, \big\| {\Psi}_{t}\big(S_{t-1}^{s_0} \big) -  \widehat{\Psi}_{t}\big(S_{t-1}^{s_0}\big)\big\|^2 \,\right]^{\frac{1}{2}} \qquad\qquad\qquad\quad   \nonumber \\
		\,+\, C\sum_{\tau=1}^{t-1} \lambda^{(t-\tau-1)} \mathbf{E}^\mu \left[ \,\big\|  {\Psi}_{\tau}\big(S_{\tau-1}^{s_0}\big) - \widehat{\Psi}_{\tau}\big(S_{\tau-1}^{s_0}\big) \big) \big\|^2 \,\right]^{\frac{1}{2}} 
\end{align}
we get a sum of one-step \gls{RNN}-based approximation errors, with each term weighted by the stability factor of $\smash{\widehat{\Psi}_t}$.

Let $S'$ be some random variable with measure $\eta'\in\mathcal{P}(\mathbb{S})$. We consider $\mathbf{E}^\mu[\|  \Psi_{\tau}(S_{\tau-1}^{s_0} ) - \widehat{\Psi}_{\tau}(S_{\tau-1}^{s_0}) \|]$ for any $\tau\in\{1,\ldots,t\}$. Applying the triangle inequality twice we get,
\begin{align}
		\mathbf{E}^\mu \left[ \,  \big\|  \Psi_{\tau}(S_{\tau-1}^{s_0} ) - \widehat{\Psi}_{\tau}(S_{\tau-1}^{s_0}) \big\|^2 \,\right]^{\frac{1}{2}} \qquad\qquad\quad \nonumber \\
		\,\leq\, \mathbf{E}^\mu \left[ \, \big\|  \Psi_{\tau}(S_{\tau-1}^{s_0} ) - {\Psi}_{\tau}(S_{\tau-1}^{S'}) \big\|^2 \,\right]^{\frac{1}{2}}~~ \nonumber\\
		 \qquad+ \mathbf{E}^\mu \left[ \, \big\|  \Psi_{\tau}(S_{\tau-1}^{S'} ) - \widehat{\Psi}_{\tau}(S_{\tau-1}^{S'}) \big\|^2 \,\right]^{\frac{1}{2}} \nonumber \\
		 \qquad + \mathbf{E}^\mu \left[ \, \big\|  \widehat{\Psi}_{\tau}(S_{\tau-1}^{S'} ) - \widehat{\Psi}_{\tau}(S_{\tau-1}^{s_0}) \big\|^2 \,\right]^{\frac{1}{2}}  \label{singlestepdecomposition}
\end{align}
From (\ref{stabilityresult1}) and (\ref{stabilityresult2}), with the latter to be verified, the first and third terms in the last inequality satisfy,
\begin{align}
	 \mathbf{E}^\mu \left[ \, \big\|  \Psi_{\tau}(S_{\tau-1}^{s_0} ) - {\Psi}_{\tau}(S_{\tau-1}^{S'}) \big\|^2 \,\right]^{\frac{1}{2}} \qquad\qquad \nonumber \\
	 +~ \mathbf{E}^\mu \left[ \, \big\|  \widehat{\Psi}_{\tau}(S_{\tau-1}^{S'} ) - \widehat{\Psi}_{\tau}(S_{\tau-1}^{s_0}) \big\|^2 \,\right]^{\frac{1}{2}} \quad\nonumber \\
	 ~ \leq~ c\, \lambda^{\tau}\,  \mathbf{E}^\mu \left[ \, \big\| s_0 - S' \big\|^2 \,\right]^{\frac{1}{2}}
\end{align}
for all $\tau\in\{1,\ldots,t\}$ and some finite $c>0$, $0<\lambda<1$. Let $C_0 := c\, \mathbf{E}^\mu[\| s_0 - S' \|^2 ]^{\frac{1}{2}}<\infty$.

It follows by a (non-trivial, but classical) result of \cite{Elton1990a} and an extension in \cite{DebalyTruquet2021a}, that given Assumption \ref{assump:stability} and stationarity of the observation sequence $(Y_t)_{t\in\mathbb{N}}$, there exists a random $S_\infty$ such that $(S_t^{\eta})_{t\in\mathbb{N}}$ with $\eta:=\mathrm{Law}(S_{\infty})\in\mathcal{P}(\mathbb{S})$ is stationary in $\mathbb{S}$.

With $\eta'=\eta$, note $\smash{S_{\tau}^{S'}}\sim\eta$, $\forall \tau\in\mathbb{N}$ by stationarity. Then, 
\begin{align}
		\mathbf{E}^\mu \left[ \, \big\| \widehat{S}_t^{{s_0}} - S_t^{s_0} \big\|^2 \,\right]^{\frac{1}{2}} 
		 \,\leq\, \left(\frac{C+1-\kappa}{1-\kappa}\right) \times \quad~~ \nonumber \\
		 \left( C_0 \,+\, \mathbf{E}^\mu \left[ \, \big\|  \Psi\big(S,Y \big) -\widehat{\Psi}\big(S,Y\big) \big\|^2 \,\right]^{\frac{1}{2}}  \right)
\end{align}
where $S$ has distribution $\eta$, and $Y$ is distributed according to the invariant law of the stationary observation sequence.

Bounding $\smash{\mathbf{E}^\mu[  \|  \Psi(S,Y ) -\widehat{\Psi}(S,Y) \|^2 ]}$ is then achieved via classical universal approximation results for feedforward neural networks. In particular, applying \citet[Theorem 4.16]{KidgerLyons2020a} we know a network exists such that,
\begin{equation}
	 \mathbf{E}^\mu \left[ \, \big\|  \Psi\big(S,Y \big) -\widehat{\Psi}\big(S,Y\big) \big\|^2 \,\right]^{\frac{1}{2}} \,\leq\, \varepsilon
\end{equation}
for any $\varepsilon>0$, with the added property that the Lipshitz constant of $\widehat{\Psi}(\cdot,y)$ is less than or equal to that of $\Psi(\cdot,y)$. The Lipschitz property follows from the main result in \cite{Eckstein2020a}, see also \cite{NeumayerGoujonBohraEtAl2022a}, when combined with \citet[Proposition 4.9]{KidgerLyons2020a} in the proof of \citet[Theorem 4.16]{KidgerLyons2020a}. 

Any $\varepsilon>0$ that solves $\epsilon \geq (\varepsilon + C_0)\,\frac{C+1-\kappa}{1-\kappa}$ is enough for the desired result (\ref{timeuniformstatement1}). 

It remains to establish the assumed (\ref{stabilityresult2}). However, this follows directly from the fact that the Lipshitz constant of the chosen $\widehat{\Psi}(\cdot,y)$ is less than or equal to that of $\Psi(\cdot,y)$. In which case, (\ref{stabilityresult2}) holds whenever (\ref{stabilityresult1}) holds. 
\end{proof}

\subsection{Discussion}

\textbf{Practical consequences}: The results above have relevant practical implications. For example, Theorem \ref{theo:maintimeuniform} implies that \gls{RNN} approximations to the optimal filter may be accurately applied on test sequences \textit{indefinitely}, even when trained on finite-length data sequences (maybe only a handful of training time steps). This is desirable in online filtering applications. In particular, the approximation errors made at each step do not accumulate unbounded over time.

Note that Theorem \ref{theo:maintimeuniform} is divorced from any training procedure and is an idealised result. In practice, one is unlikely to know what a suitable sufficient statistic looks like and may only be interested in finding a small number of optimal conditional moments (e.g. maybe just the mean and covariance). A loss function used in practice for training a network might then only account for a small subset of the sufficient statistics required to define an optimal filtering recursion. This does not limit the theory. The recursion on the (wide) final activation layer may propagate a much higher-dimensional statistic defining the optimal filter than is carried forward through a linear output layer. In practice one may also consider a deeper feedforward network appended after the feedback layer, so as to compute other conditional functions. While these design and training ideas are not pertinent to the stated approximation capability results, they may be relevant in practical network design and in the design of loss functions and (e.g. hierarchical) training schemes.

The feedback in (\ref{neuralproofstruct1}), (\ref{neuralproofstruct2}), (\ref{neuralproofstruct3}), used in the proof of Theorem \ref{theorem:basicapprox}, acts to memorise the observations, and the network otherwise seeks to approximate (non-recursive) solutions to (\ref{optimalcriterion}), as in (\ref{filteringpointestimate}). Conversely, the time-uniform result in Theorem \ref{theo:maintimeuniform} is based on a network that directly approximates the optimal filtering recursion itself. This distinction offers some insight into the contrast between memorisation versus recursive state feedback in \gls{RNN}s. Contrasting these two structures may influence network design and training in some applications; e.g. in models with long time-dependencies, or if the signal is not well understood, or perhaps to counter the effects of finite truncation in the number of sufficient statistics. Although the network (\ref{neuralproofstruct1}), (\ref{neuralproofstruct2}), (\ref{neuralproofstruct3}) involves a memorisation over the entire finite interval of interest, in practice $\mathbf{E}^\mu\left[ \rho(X_\tau) \,|\, Y_\tau, \cdots, Y_1\right]$ may depend only weakly on observations in the distant past, e.g. as recognised and exploited in so-called fixed-lag smoothing methods, see \cite{Moore1973a}. Thus, a smaller network, e.g. for (\ref{finitenewtargetf}), with a combination of recursion and memorisation may provide good longer-time performance.

\textbf{Limitations}: Unfortunately, (\ref{timeuniformstatement1}) cannot be expected to hold, in general domains, under much weaker assumptions. For example, consider a linear-Gaussian signal/observation model that is controllable and observable, but with an unstable latent signal. The Kalman filter is stable in this case and Assumption \ref{assump:stability}, i.e. (\ref{stabilityresult1}), holds. However, in this case, the transition of $X_t$ is not ergodic and $\mathbf{P}^\mu(Y_t\in\cdot)$ places most mass on sets of ever increasing distance from the origin. Thus, Assumption \ref{assump:stationary} does not hold, and its not possible to uniformly control the one-step approximation error as needed in the latter parts of the proof of Theorem \ref{theo:maintimeuniform}. 

In general, there are limited circumstances in which a finite-dimensional recursive filter of sufficient statistics exists. Even if such a filter exists, the square-integrability assumption on the sufficient statistic itself may be limiting. However, there may be natural finite-dimensional truncations of infinite-dimensional representations in the form (\ref{suffstat-filter}), e.g. see \cite{Goodman1975a,Marcus1979a}. For example, a Taylor-series truncation of the optimal filter, up to any finite order, can be approximated with this framework (with the most basic truncation being the so-called \gls{EKF}, see \cite{Goodman1975a,AndersonMoore1979a}, based on a first-order Taylor expansion).

Since the filter is in practice learnt from data, and large neural networks are seemingly able to accurately capture complex high-dimensional maps, a \gls{RNN} algorithmic framework offers distinct advantages in Bayesian filtering and time-series inference (e.g. over existing methods like model-based filters like the \gls{EKF}/\gls{UKF}, or Monte Carlo/particle methods). For example, a \gls{RNN} may capture much higher-order finite truncations of some infinite optimal filtering statistic than is possible in other approaches, with less (test-time) computational burden. Conversely, Theorem \ref{theo:maintimeuniform} sheds some light on the theoretical limitations, viz. Assumption \ref{assump:stationary}, of (at least naive versions of) this method in approximating recursive filter models; though contrast these with Theorem \ref{theorem:basicapprox} on finite intervals. 

The broad approach to prove Theorem \ref{theo:maintimeuniform} follows from results in the particle filtering community due to \cite{DelMoralGuionnet2001a}. Related work is focused on the transfer of stability from the optimal filter to the (particle) approximations; with deep technical methods aimed at proving and applying conditions like in Assumption \ref{assump:stability}, see, e.g., \cite{HeineCrisan2008a,Handel2009c,Whiteley2013a,DoucMoulinesOlsson2014a} and the stability literature referenced earlier. See also \cite{CrisanLopezYelaMiguez2020a}. The setup and proof of Theorem \ref{theo:maintimeuniform} takes a simpler view in terms of Assumption \ref{assump:stability}, so as to maximise the pedagogical value of these proof methods in problems involving \gls{RNN}s, where they may not be known. This approach also offers some insight into the specific formulation and obstacles native to the generic neural approximation, e.g. as per Assumption \ref{assump:stationary}. The details here contrast with other approximations: e.g. with particle approximations the one-step error may be controlled easier over longer horizons without Assumption \ref{assump:stationary}.
However, general theoretical results with particles may still be limited by comparable assumptions, as in \cite{Handel2009c,DoucMoulinesOlsson2014a}. Indeed, in \cite{BishopDelMoral2023} it is proven that the basic bootstrap particle filter of \cite{GordonSalmondSmith1993a} is incapable of tracking an unstable linear-Gaussian signal; in which setting Assumption \ref{assump:stationary} fails to hold, although Assumption \ref{assump:stability} still holds.

\textbf{Extensions}: A trained \gls{RNN}-based estimator will apply in practice under more complicated models than posited earlier; e.g., non-independent, non-Markov or time-varying. 

Exponential ergodicity of the latent signal is in general settings likely sufficient for Theorem \ref{theo:maintimeuniform} to hold; i.e. this property implies asymptotic stationarity of the observations, and likely transfers to the filter so conditions like Assumption \ref{assump:stability} hold. Next we study a special case in which Assumption \ref{assump:stationary} can be relaxed to just asymptotic stationarity.

\subsection{Extension to Asymptotically Stationary Observations: Kalman Filtering Case Study}

We examine a mild, but useful relaxation of Assumption \ref{assump:stationary}; and illustrate where problems may arise when this assumption does not hold. Consider a linear-Gaussian model,
\begin{align}
	X_{t}^\mu \,&=\, \mathbf{F}X_{t-1}^\mu + V_t \label{sysLintheory} \\
	Y_{t} \,&=\, \mathbf{H}X_{t} + W_t \label{measLintheory}
\end{align}
where $\mathbf{F}$ and $\mathbf{H}$ are real matrices, and $X_{0}$ has Gaussian distribution $\mu\in\mathcal{P}(\mathbb{X})$ with mean $\overline{X}_0$ and covariance $\mathbf{C}_0$. The random sequences $(V_t)_{t\in\mathbb{N}}$, $(W_t)_{t\in\mathbb{N}}$ are mutually independent zero-mean Gaussian with covariance matrices $\mathbf{Q}\geq0$ and $\mathbf{R}>0$, also independent of $X_0$. If the model is detectable and stabilisable, see \cite{AndersonMoore1979a}, then Assumption \ref{assump:stability} holds. We assume even more later.

We define the filter state $\smash{S^{(\overline{X}_{0},\mathbf{C}_{0})}_t}:=(\overline{X}_{t},\mathbf{C}_{t})\in\mathbb{S}$ via a slight abuse of notation. The form of the true Kalman filter for the model (\ref{sysLintheory}), (\ref{measLintheory}) is given by,
\begin{align}
	\overline{X}_{t} &\,:=\ \mathbf{E}^\mu\left[ X_t \,|\, Y_1, \cdots, Y_t\right] \nonumber\\ 
	&\,=\,  \left(\mathbf{F} - \mathbf{K}_{t}^*\mathbf{H}\mathbf{F} \right) \overline{X}_{t-1} \,+\, \mathbf{K}_{t}^* Y_{t}  \label{kalman-proof-theory1}
\end{align}
\begin{align}
	\mathbf{C}_{t} &\,:=\ \mathbf{Cov}^\mu\left[ X_t \,|\, Y_1, \cdots, Y_t\right] \,=\, \mathbf{Cov}^\mu\left[ X_t  - \overline{X}_{t} \right] \nonumber\\
	&\,=\, \left(\mathbf{I} - \mathbf{K}_{t}^*\mathbf{H} \right)\left(\mathbf{F}\mathbf{C}_{t-1}\mathbf{F}^\top + \mathbf{Q}\right) \label{kalman-proof-theory2}
\end{align}
where $\mathbf{Cov}^\mu[\cdot]$ denotes the relevant covariance, and,
\begin{align}
	\mathbf{K}_{t}^* \, =\,  \left(\mathbf{F}\mathbf{C}_{t-1}\mathbf{F}^\top + \mathbf{Q}\right)\mathbf{H}^\top\,\times \qquad\qquad\qquad \nonumber \\
		\left[\mathbf{R} + \mathbf{H}\left(\mathbf{F}\mathbf{C}_{t-1}\mathbf{F}^\top + \mathbf{Q}\right)\mathbf{H}^\top\right]^{-1} \label{kalmangain1}
\end{align}

\begin{corollary} \label{corollarylinearKalman}

Consider a linear-Gaussian signal and observation model (\ref{sysLintheory}), (\ref{measLintheory}). Assume the spectral radius of $\mathbf{F}$ is strictly $<1$. Let $\widehat{\Psi}^\theta(\mathbf{s}_{t-1}, Y_t)$ denote a \gls{RNN}, as in (\ref{generalRNN1}), (\ref{generalRNN2}), (\ref{generalRNN3}), taking as input elements in the sequence $(Y_t)_{t\in\mathbb{N}}$. Then for any  $\epsilon>0$, there is a finite real parameter $\theta^*$, and an initialisation vector $\mathbf{s}_{0}$, such that,
\begin{equation}
	 \mathbf{E}^\mu\left[\, \left\| \smash{S^{(\overline{X}_{0},\mathbf{C}_{0})}_t} - \widehat{\Psi}^{\theta^*}(\mathbf{s}_{t-1}, Y_t) \right\|^2 \right]^{\frac{1}{2}} ~\leq~ \epsilon
	 \label{kalmanapproxresult}
\end{equation}
for all $t\in\mathbb{N}$ and all $(\overline{X}_{0},\mathbf{C}_{0})\in\mathbb{S}$
\end{corollary}

The observation sequence is neither Markov nor stationary in general, and the hypotheses of the corollary here do not call for these conditions. 

\begin{proof}
	Stacking (\ref{kalman-proof-theory1}) and \eqref{measLintheory} we may write, 
\begin{align}
	\left[\begin{array}{c} \overline{X}_{t} \\ X_t  \end{array}\right] \,&=\, \left[\begin{array}{cc} \mathbf{F} - \mathbf{K}_{t}^*\mathbf{H}\mathbf{F}  & \mathbf{K}_{t}^*\mathbf{H}\mathbf{F} \\ \mathbf{0} & \mathbf{F}  \end{array}\right] \left[\begin{array}{c} \overline{X}_{t-1} \\ X_{t-1}  \end{array}\right] \,\nonumber\\
	& \qquad +\, \left[\begin{array}{cc}  \mathbf{K}_{t}^*  & \mathbf{0} \\ \mathbf{0} & \mathbf{I}  \end{array}\right] \left[\begin{array}{c} {W}_{t} \\ V_t  \end{array}\right] \label{stackedsystem}
\end{align}
The transition matrix for the given stacked system has spectral radius $<1$ (its eigenvalues are the union of those of the diagonal blocks). We consider a \gls{RNN} approximation of (\ref{stackedsystem}) and (\ref{kalman-proof-theory2}), treating this pair as the filter of interest. Assumption \ref{assump:stability} holds for the pair (\ref{stackedsystem}), (\ref{kalman-proof-theory2}). The added noise process for the stacked system is stationary and thus Assumption \ref{assump:stationary} holds. Theorem \ref{theo:maintimeuniform} can now be applied. That is, there is a \gls{RNN} that may approximate the system (\ref{stackedsystem}), (\ref{kalman-proof-theory2}) with $(W_t,V_t)$ as an input. Suppose that we substitute $Y_t$ for the input $W_t$, drop the input $V_t$, and set to zero the affine feedback map of $\widehat{X}_t$. The result is a sub-network approximation of the desired (\ref{kalman-proof-theory1}), (\ref{kalman-proof-theory2}). By enforcing an accurate enough network approximation of (\ref{stackedsystem}), the result (\ref{kalmanapproxresult}) holds for the sub-network approximation of the desired (\ref{kalman-proof-theory1}) when swapping the input from $W_t$ to $Y_t$. This follows because,
\begin{align}
	\mathbf{E}^\mu[\|Y_t-W_t  \|^2]^{\frac{1}{2}} &\leq \mathbf{E}^\mu[\|\mathbf{H}\mathbf{F}^t\,X_0 + \mathbf{H}\textstyle{\sum_{k=1}^{t-1}} \mathbf{F}^kV_k  \|^2]^{\frac{1}{2}} \nonumber \\
		&\leq \|\mathbf{H}\|\|\mathbf{F}^t\|\,(\mathrm{trace}(\mathbf{C}_0) + \|\overline{X}_0\|^2)^{\frac{1}{2}} \nonumber \\
		& \qquad + \frac{\|\mathbf{H}\|\, \mathrm{trace}(\mathbf{Q})^{\frac{1}{2}}}{1 - \|\mathbf{F}\|}
\end{align}
is just a finite constant (and $\|\mathbf{F}^t\|$ even goes to zero).
\end{proof}

We remark that it is just the one model property (i.e. the spectral radius of $\mathbf{F}$ being $<1$) that ensures the observation sequence is asymptotically stationary, and that Assumption \ref{assump:stability} is satisfied. This supports the more general notion that exponential ergodicity of the latent signal is sufficient for Theorem \ref{theo:maintimeuniform} to hold, as noted just before this subsection. Note Assumption \ref{assump:stability} holds under just model detectability and stabilisability conditions \citep{AndersonMoore1979a} (which can hold when the signal is unstable; but which are implied when the spectral radius of $\mathbf{F}$ is $<1$).

We have mildly relaxed the Assumption \ref{assump:stationary} that $(Y_t)_{t\in\mathbb{N}}$ is stationary used in Theorem \ref{theo:maintimeuniform}. To see where things may go wrong, consider an identity approximation of a scalar observation $Y_t$ given by $\sigma(Y_t + b)-b$ for a sufficiently large scalar bias $b>0$. If $Y_t$ is stationary, then there is a $b$ such that $\mathbf{E}^\mu[\|Y_t - (\sigma(Y_t + b)-b)\|^2]\leq\varepsilon$ for any $\varepsilon>0$. However, suppose $X_t$ is unstable and $Y_t$ moves to the negative; then there is a time $\tau= \tau(b)\in\mathbb{N}$ such that with exponentially large probability we have $\sigma(Y_t + b)-b=-b$ for all $t>\tau$ . The approximation error for all $t>\tau$ grows on average and a time-uniform bound is impossible.

\section{EXAMPLE}

We consider a simple scalar example to illustrate the results, conditions, and some of the discussion points. Let,
\begin{align}
	X_t \,&=\, \alpha\,X_{t-1} + V_t,~\quad\nonumber\\
	Y_t \,&=\, X_t +\beta\, W_t
\end{align}
where $\mathbb{X}=\mathbb{R}$, $\alpha\in\{0.98,1.001\}$, $\beta\in\{1,2\}$ and $V_t$, $W_t$ are independent standard Gaussian white noises, independent of $X_0$ which is zero mean Gaussian with variance $25$. 

The Kalman filter $\smash{S^{(\overline{X}_{0},\mathbf{C}_{0})}_t}:=(\overline{X}_{t},\mathbf{C}_{t})$ is optimal and Assumption \ref{assump:stability} holds in all cases considered (by satisfaction of the classical detectability and stabilisability conditions). If $|\alpha|<1$, then $X_t$ is ergodic and Assumption \ref{assump:stationary} holds. Here, $|\alpha|<1$ is an explicit provable threshold on the satisfaction of Assumption \ref{assump:stationary}. We also have the ground-truth comparison in the optimal Kalman filter. We can thus study the requisite condition for time-uniform approximation (e.g. moving from $\alpha=0.98$ to just $\alpha=1.001$).

We train a \gls{RNN} based on (\ref{empiricalcostneural}) with $\rho(x)=x$, i.e. neglecting the variance. 
We use $N_{\mathrm{train}}$ samples over horizons of length $T_{\mathrm{train}}$. This training loss function is consistent with prior discussion noting that in practice one is unlikely to know the nature of sufficient statistics, nor desire them all as outputs, nor be willing to train on loss functions with all of them. We also consider a basic particle filter implementation \citep{GordonSalmondSmith1993a} with $1000$ particles. 

We compute $\smash{(\tfrac{1}{N_{\mathrm{test}}}\sum_{n=1}^{N_{\mathrm{test}}} ( \overline{X}^{(n)}_{t}- \widehat{X}^{(n)}_{t} )^2)^{\frac{1}{2}}}$ where $\smash{\widehat{X}^{(n)}_{t}}$ is the $n$-th mean state estimate of the particle or \gls{RNN} approximation from $N_{\mathrm{test}}=1000$ independent test examples. The test horizon is $T_{\mathrm{test}}=2000$ in all cases. 

In Fig \ref{fig:rmse1} ($\alpha=0.98$, $\beta\in\{1,2\}$) and Fig \ref{fig:rmseNN1} ($\alpha=1.001$, $\beta=2$) we plot the errors for different parameters. For each case we consider a \gls{RNN} in the form (\ref{RNNrecursiveform1}), (\ref{RNNrecursiveform2}), (\ref{RNNrecursiveform3}) with $L=3$ and $7$ \textsc{ReLU}s on both hidden layers. 

In each case in Fig \ref{fig:rmse1} we train with just $T_{\mathrm{train}}=20$, and $N_{\mathrm{train}}=5000$. With $\alpha=0.98$, Corollary \ref{corollarylinearKalman} holds, and the \gls{RNN}-based method outperforms the particle filter (which does not handle accurate observations $\beta=1$ well here). 

\begin{figure*}[!ht]
    \centering
    \includegraphics[width=0.875\textwidth]{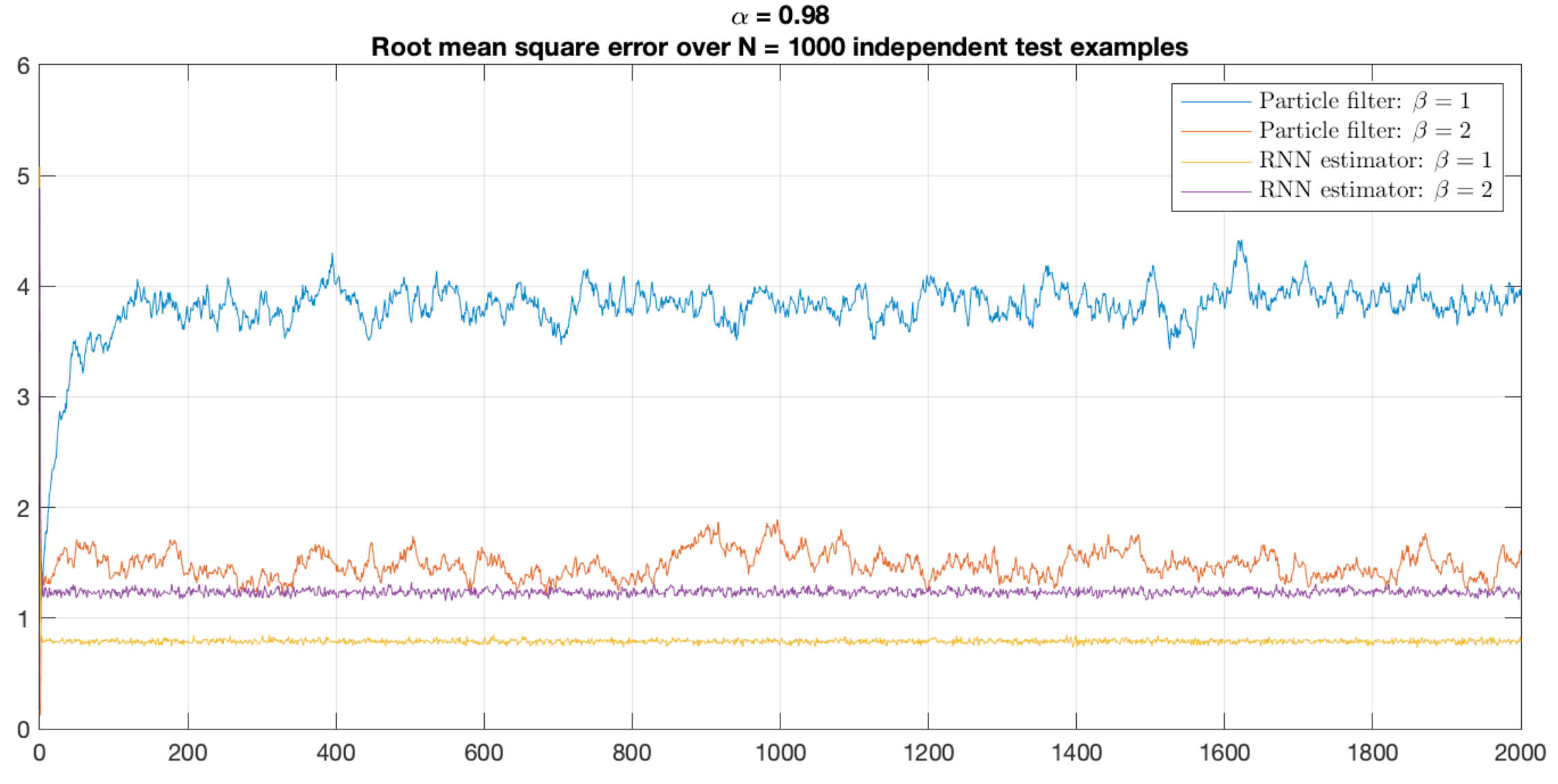}
        \caption{Root mean square errors.  }
    \label{fig:rmse1}
\end{figure*}

In Fig \ref{fig:rmseNN1} we consider $T_{\mathrm{train}}=20$, $N_{\mathrm{train}}=5000$ (as before), and also $T_{\mathrm{train}}=200$, $T_{\mathrm{train}}=1000$, $T_{\mathrm{train}}=2000$ (in the latter two cases we reduce $N_{\mathrm{train}}=1000$).

\begin{figure*}[!ht]
    \centering
     \includegraphics[width=0.875\textwidth]{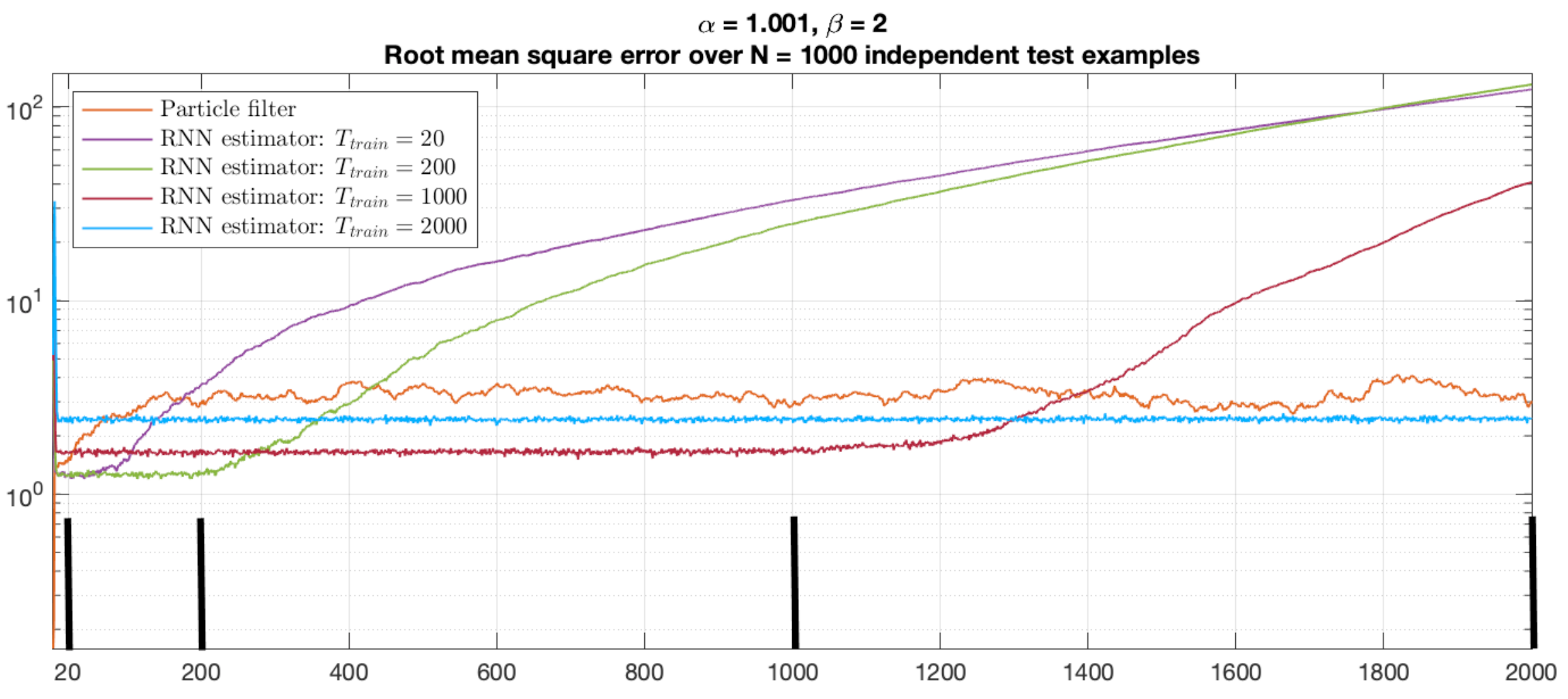}
    \caption{Root mean square errors (log vertical scale). }
    \label{fig:rmseNN1}
\end{figure*}

These examples reflect on a number of discussion points: 1) training on an output/cost function of interest, but of lower dimension than the sufficient statistic required to define the optimal filter; 2) when the conditions of Theorem \ref{theo:maintimeuniform}/Corollary \ref{corollarylinearKalman} are met, we may train on finite-length (often very short, maybe one or a handful of steps, depending on the ergodicity) but apply the filter indefinitely at test time with no unbounded accumulated growth of error; 3) differing capabilities and particulars of differing approximation schemes (e.g. particle vs neural in this case); 4) if either assumption is unmet, we may learn networks that work well on the length of the training data, e.g. as in Theorem \ref{theorem:basicapprox}, but eventually the error may start to accumulate.

\section{CONCLUSION}

We consider a generic recurrent neural network framework that approximates directly a recursive mapping from observational inputs to some desired Bayesian filter statistics. The focus of this article is the approximation capability of this framework. We provide approximation error bounds for filtering in general non-compact domains. The main result of this work is a couple of conditions on the underlying observation sequence and on the optimal filter that when satisfied allow one to approximate the optimal Bayesian filter to any desired accuracy over an indefinite or infinite time horizon. When applicable, this strong, time-uniform approximation result ensures good long-time filtering performance in practice. We discuss and illustrate a number of practical concerns and implications of these results; and we contrast the fixed horizon and time-uniform results with each other. With respect to the time-uniform universal approximation result, we explore the mechanisms by which the required conditions manifest in practice, their necessity, and we discuss their appearance in similar results based on different filtering approximation schemes.

\end{document}